%% file: EDM_Article_Submission.tex
\documentclass{edm_article}

\usepackage{textcomp}
\usepackage{soul}
\usepackage{hyperref}
\usepackage{url}
\usepackage{caption}
\usepackage{subcaption}
\input{preamble}

\usepackage{amsthm}
\usepackage{booktabs}
\usepackage{xcolor}

\newtheorem{prop}{Proposition}

\begin{document}

\flushbottom

\title{VarFA: A Variational Factor Analysis Framework \\ For Efficient Bayesian Learning Analytics
}
%
\numberofauthors{4}

%
%
%
%

\author{Zichao Wang$^{1}$, Yi Gu$^{2}$, Andrew S. Lan$^{3}$, Richard G. Baraniuk$^{1,4}$ \\
\affaddr{$^{1}$Rice University, $^2$Northwestern University, $^3$University of Massachusetts Amherst, $^4$OpenStax} \\
{\fontsize{0.5}{2}\email{jzwang@rice.edu}, \email{Yi.Gu@u.northwestern.edu}, \email{andrewlan@cs.umass.edu}, \email{richb@rice.edu}}
%
%
}

\maketitle


\begin{abstract}
We propose VarFA, a variational inference factor analysis framework that extends existing factor analysis models for educational data mining to efficiently output uncertainty estimation in the model's estimated factors. Such uncertainty information is useful, for example, for an adaptive testing scenario, where additional tests can be administered if the model is not quite certain about a students' skill level estimation. Traditional Bayesian inference methods that produce such uncertainty information are computationally expensive and do not scale to large data sets. VarFA utilizes variational inference which makes it possible to efficiently perform Bayesian inference even on very large data sets. We use the sparse factor analysis model as a case study and demonstrate the efficacy of VarFA on both synthetic and real data sets. VarFA is also very general and can be applied to a wide array of factor analysis models. Code and instructions to reproduce results in this paper are available from~{\color{blue}{\url{https://tinyurl.com/tvm4332}}}.
\end{abstract}

%


\flushbottom
\section{Introduction}
\label{sec:intro}
A core task for many practical educational systems is \textit{student modeling}, i.e., estimating students' mastery level on a set of skills or knowledge components (KC)~\cite{vanlehn1988student, chrysafiadi2013student}. Such estimates allow in-depth understanding of students' learning status and form the foundation for automatic, intelligent learning interventions.
For example, intelligent tutoring systems (ITSs)~\cite{shute1994intelligent, pls1,anderson1995cognitive,nye2014autotutor,heffernan2014assistments,rafferty2015inferring} rely on knowing the students' skill levels in order to effectively recommend individualized learning curriculum and improve students' learning outcomes. In the big data era, student modeling is usually formulated as an educational data mining (EDM) problem~\cite{romero2010educational,baker2009state,merceron2005educational} where an underlying machine learning (ML) model estimates students' skill mastery levels from students' learning records, i.e., their answers to assessment questions. 

\sloppy
Many student modeling methods have been proposed in prior literature. A fruitful line of research for student modeling follows the \textit{factor analysis (FA)} approach.
FA models usually assume that an unknown, potentially multi-dimensional student parameter, in which each dimension is associated with a certain skill, explains how a student answers questions and is to be estimated. Popular and successful FA models include item response theory (IRT)~\cite{irt}, multi-dimensional IRT~\cite{mirt}, learning factor analysis (LFA)~\cite{LFA}, performance factor analysis (PFA)~\cite{PFA}, DASH (short for {d}ifficulty, {a}bility, and {s}tudent {h}istory)~\cite{lindsey2014improving,mozer2016predicting}, DAS3H (short for {d}ifficulty, {a}bility, {s}kill and {s}tudent {s}kill {h}istory)~\cite{das3h}, knowledge tracing machines (KTM)~\cite{vie2019knowledge}, and so on. 
Recently, more complex student models based on deep neural networks (DNN) have also been proposed~\cite{DKT,zhang2017dynamic,minn2018deep}. Nevertheless, thanks to their simplicity, effectiveness and robustness, FA models remain widely adopted and investigated for practical EDM tasks. Moreover, there is evidence that simple FA models could even outperform DNN models for student modeling in terms of predicting students' answers~\cite{irt-best}. Because of FA models' competitive performance and its elegant mathematical form, we focus on FA-based student models in this paper.

Most of the aforementioned FA models compute a single point estimate of skill levels for each student. Often, however, it is not enough to obtain mere point estimates of students' skill levels; knowing the model's uncertainty in its estimation is crucial because it potentially helps improve the model's performance and improve both students' and instructors' experience with educational systems. 
For example, in ITS, an underlying model can use the uncertainty information in its estimation of students' skill level to automatically decide that its recommendations based on highly uncertain estimations are unreliable and instead notify a human instructor to evaluate the students' performance. This enables collaboration between ITS and instructors to create a more effective learning environment.
%
In adaptive testing systems~\cite{chang2009nonlinear,yan2016computerized}, 
knowing the uncertainty in model's estimation could help the model intelligently pick the next test items to most effectively reduce its uncertainty about estimated students' skill levels. This will help to potentially reduce the number of items needed to have a confident, accurate estimation of the students' skill mastery level, saving time for both students to take the test and instructors to have a good assessment of the student's skills.
%
%

All the above applications require the model to ``know what it does not know,'' i.e., to quantify the uncertainty of its estimation. Achieving this does not necessary changes the model. rather, we need a different inference algorithm for inferring not only a point estimate of student's skill level from observed data (i.e., students' answer records to questions) but also uncertainty in the estimations. 

Fortunately, there exist methods that both compute point estimates of students' skill levels and quantifies the uncertainty of those estimates. These methods usually follow the Bayesian inference paradigm, where each student's unknown skill levels are treated as random variables drawn from a \textit{posterior distribution}. Thus, one can use the \textit{credible interval} to quantify the model's uncertainty on the student's estimated skill level. A classical method for Bayesian inference is Monte Carlo Markov Chain (MCMC) sampling~\cite{gelman2013bayesian} which has been widely used in many other disciplines~\cite{gilks1995markov} other than EDM. In the context of student modeling with FA models, existing works such as ~\cite{lan14a,fox2010bayesian,pardos2010using,merkle2011comparison} have applied MCMC methods to obtain credible intervals. 

Unfortunately, classic Bayesian inference methods suffer from extensive computational complexity. For example, each computation step in MCMC involves a time-consuming evaluation of the posterior distribution. Making matters worse, MCMC typically takes many more steps to converge than non-Bayesian inference methods. As a concrete illustration, in~\cite{lan14a}, the Bayesian inference method (10 minutes) is about 100 times slower than the other non-Bayesian inference method based on stochastic gradient descent (6 seconds). The high computational cost prevents Bayesian inference methods from mass-deployment in large-scale, real-time educational systems where timely feedback is critical for learning~\cite{driscoll1994psychology} and tens of thousands of data points need to be processed in seconds or less instead of minutes or hours.
It is thus highly desirable to accelerate Bayesian inference so that one can quantify model's uncertainty in its estimation as efficiently as non-Bayesian inference methods.

\paragraph{Contributions}
In this work, we propose VarFA, a novel framework based on \textit{variational inference} (VI) to perform efficient, scalable Bayesian inference for FA models. The key idea is to approximate the true posterior distribution, whose costly computation slows down Bayesian inference, with a \textit{variational distribution}. We will see in Section~\ref{sec:method} that, with this approximation, we turn Bayesian inference into an optimization problem where we can use the same efficient inference algorithms as in non-Bayesian inference methods. Moreover, the variational distribution is very flexible and we have full control specifying it, allowing us to freely use the latest development in machine learning, e.g., deep neural networks (DNNs), to design the variational distribution that closely approximates the true posterior. 
Thus, we also regard our work as a first step in applying DNNs to FA models for student modeling, achieving efficient Bayesian inference (enabled by DNNs) without losing interpretability (brough by FA models).
We demonstrate the efficacy of our framework on both synthetic and real data sets, showcasing that VarFA substantially accelerates classic Bayesian inference for FA models with no compromise on performance. 

The remainder of this paper is organized as follows. Section~\ref{sec:related-work} introduces the problem setup in FA and reviews the important FA models and their inference methods, in particular Bayesian inference. Section~\ref{sec:method} explains our VarFA framework in detail. Section~\ref{sec:experiments} presents extensive experimental results that substantiate the claimed advantages of our framework. Section~\ref{sec:conclusion} concludes this paper and discusses possible extensions to VarFA.

\section{Background and Related Work}
\label{sec:related-work}
We first set up the problem and review related work. Assume we have a data set $\mY\in\mathbb{R}^{N\times Q}$ organized in matrix format where $N$ is the total number of students and $Q$ is the number of questions. This is a binary students' answer record matrix where each entry $y_{ij}$ represents whether student $i$ correctly answered question $j$. Usually, not all students answer all questions. Thus, $\mY$ contains missing values. We use $\{i,j\}\in\Omega_{\rm obs}$ to denote entries in $\mY$, i.e., the $i$-th student's answer record to the $j$-th question, that are observed. 


We are interested in models capable of inferring each $i$-th student's skill mastery level that can accurately predict the student's answers given the above data. These models are often evaluated on the prediction accuracy and whether the inferred student skill mastery levels are easily interpretable and educationally meaningful. We now review factor analysis models (FA), one of the most widely adopted and successful methodologies for the student modeling task.

\subsection{Factor Analysis For Student Modeling}
\label{sec:fa}
One of the earliest FA model for student modeling is based on the item response theory (IRT)~\cite{irt}. It usually has the following form
\begin{align}
    \begin{split}
        \mathbb{P}(y_{ij}=1) = \sigmoid(c_i + \mu_j)\,,\label{eq:irt}
    \end{split}
\end{align}
which assumes that each student's answer $y_{ij}$ is independently Bernoulli distributed. The above formula says that the students' answers can be explained by an unknown scalar student skill level factor $c_i$ for each student $i$ and an unknown scalar question difficulty level factor $\mu_j$ for each question $j$. $\sigmoid(\cdot)$ is a Sigmoid activation function, i.e., $\sigmoid(x) = 1 / (1 + {\rm exp}(-x))$. The multi-dimensional IRT model (MIRT)~\cite{mirt} extends IRT by using a multi-dimensional vector to represent the student skill levels. More recently,~\cite{lan14a} proposed sparse factor analysis model (SPARFA) which extends MIRT by imposing additional assumptions, resulting in improved interpretations of the inferred factors. 

Other FA models seek to improve student modeling performance by cleverly incorporate additional auxiliary information. 
For example, the additive factor model (AFM)~\cite{LFA} incorporates students' accumulative correct answers for a question and the skill tags associated with each question:
\begin{align}
    \begin{split}
        \mathbb{P}(y_{ij}=1) = \sigmoid\left(c_i + \sum_{k=1}^\kappa\big( q_{kj}\beta_k + q_{kj}\rho_k b_{ik} \big)\right)\,,\label{eq:laf}
    \end{split}
\end{align}
where $\kappa$ is the total number of skills in the data, $b_{ik}$ is the $i$-th student's total number of correct answers for skill $k$ and $q_{kj}\in\{0,1\}$ indicates whether the skill $k$ is associated with question $j$. In particular, $q_{kj}$'s form matrix $\mQ\in\mathbb{R}^{K\times Q}$ which is commonly known as the Q-matrix in literature~\cite{qmatrix}. The unknown, to-be-inferred factors are $\beta_k$, a scalar difficulty factor for each skill $k$, and $\rho_k$, a scalar learning rate of skill $k$.
The performance factor model (PFM)~\cite{PFA} builds upon AFM that additionally incorporate the total number of a student's incorrect answers to questions associated with a skill. The instructor factor model (IFM)~\cite{chi2011instructional} further builds on PFM to incorporate the prior knowledge of whether a student has already mastered a skill. More recently, ~\cite{vie2019knowledge} introduces knowledge tracing machines (KTM) that could flexibly incorporate a number of auxiliary information mentioned above, thus generalizing AFM, PFM and IFM. 

Another way to improve student modeling performance is to utilize the so-called \textit{memory}, i.e., using students' historic action data over time. For example, ~\cite{lindsey2014improving,mozer2016predicting} proposed DASH (short for \underline{d}ifficulty, \underline{a}bility, and \underline{s}tudent \underline{h}istory) that incorporates a student's total number of correct answers and total number of attempts for a question in a given time window
\begin{align}
    \begin{split}
        \mathbb{P}(y_{ij}=1) &= \sigmoid\Bigg( c_i - \mu_j + \sum_{w=0}^{W-1} \Big(\theta_{2w+1}{\rm log}(1+\rho_{ijw}) \\[-5pt]
        &\qquad\qquad\qquad\qquad- \theta_{2w+2}{\rm log}(1+\gamma_{ijw})\Big) \Bigg)
    \end{split}\label{eq:dash}
\end{align}
where $W$ is the length of the time windows (e.g., number of days), $\rho_{ijw}$ and $\gamma_{ijw}$ are the $i$-th student's total number of correct answers and total number of attempts for question $j$ at time $w$, respectively. $\theta$'s are parameters that captures the effect of correct answers and total attempts. More recently, DAS3H~\cite{das3h} extends DASH to further include skill tags associated with each question which essentially amounts to adding, within the last summation term in Eq.~\ref{eq:dash}, another summation over the number of skills.

\subsubsection{A general formulation for FA models}
Although the above FA models differ in their formulae, modeling assumptions and the available auxiliary data used, we argue that the aforementioned FA models can be unified into a canonical formulation below 
\begin{align}
    \begin{split}
        \mathbb{P}(y_{ij}=1) = \sigmoid(\rvc_i^\top \rvm_j + \mu_j)\,, \label{eq:sparfa}
    \end{split}
\end{align}
where $\rvc_i\in\mathbb{R}^K$, $\rvm_j\in\mathbb{R}^K$ and $\mu_j\in\mathbb{R}$ are factors whose dimension, interpretations and subscript indices depend on the specific instantiations of the FA model. To illustrate that Eq.~\ref{eq:sparfa} subsumes the FA models mentioned above, we demonstrate, as examples, how to turn SPARFA and AFM into the form in Eq.~\ref{eq:sparfa} and how to reinterpret the parameters in the reformulation. The equivalence between Eq.~\ref{eq:sparfa} and the original SPARFA formula is immediate and we can interpret the parameters in Eq.~\ref{eq:sparfa} as follows to recover SPARFA: $K$ is the number of \textit{latent skills} that group the actual skill tags in the data set into meaningful coarse clusters; $\rvc_i$ represents the $i$-th student's skill level on the latent skills; $\rvm_j$ represents the strength of association of the $j$-th question with the latent skills which is assumed to be nonnegative and sparse for improved interpretation; and $\mu_j$ represents the difficulty of question $j$. All three factors in SPARFA are assumed to be unknown.

Similarly, for AFM, we can perform the following change of variables and indices to obtain Eq.~\ref{eq:sparfa}. We first change the indices $\rvm_j$ to $\rvm_{ij}$ and $\mu_j$ to $\mu_i$, and remove the index in $\rvc_i$ to $\rvc$. Then, we simply set $\rvc = [\beta_1, ..., \beta_\kappa, \rho_1, ..., \rho_\kappa]^\top$, $\rvm_{ij} = [q_{1j},...,q_{\kappa j}, q_{1j}b_{ik},...,q_{\kappa j}b_{i\kappa}]^\top$ and $\mu_{i} = \alpha_i$ to obtain Eq.~\ref{eq:sparfa}. We can interpret the parameters in the reformulation as follows to recover AFM: $\kappa=K/2$ is the total number of skill tags in the data; $\rvc$ represents the unknown skill information including skill difficulty and learning rate; $\rvm_{ij}$ summarizes known auxiliary information for the $i$-th student and $j$-th question; $\mu_i$ represents the unknown $i$-th student's skill mastery level. Tthe other FA models can be cast into Eq.~\ref{eq:sparfa} in a similar fashion. Note that, if the observed data $\mY$ is not binary but rather continuous or categorical, we can simply use a Gaussian or categorical distribution for the observed data and change the Sigmoid activation $\sigmoid(\cdot)$ to some other activation functions accordingly. In this way, we still retain the general FA model in Eq.~\ref{eq:sparfa}. We will use this canonical form to facilitate discussions in the rest of this paper.

\subsection{Inference Methods for FA Models}
\label{sec:FA-inference}
We first introduce the inference objective and briefly review maximum log likelihood estimation method. Then, we review existing works that uses Bayesian inference for FA and highlight their high computational complexity, paving the way to VarFA, our proposed efficient Bayesian inference framework based on variational inference.

The factors in Eq.~\ref{eq:sparfa} may contain known factors that need no inference. Therefore, for convenience of notation, let $[\nu, \theta, \psi]$ be a partition of the factors $\rvc_i$, $\rvm_j$ and $\mu_j$ where $\nu$ contains the unknown students' skill level factors to be estimated, $\theta$ contains the remaining unknown factors and $\psi$ contains the known factors. For example, for AFM, $\nu = \{\mu_1,...,\mu_N \}$, $\theta=\rvc$ and $\psi = \{\rvm_{11}, ..., \rvm_{NQ}\}$. For SPARFA, $\nu = \{\rvc_1,...,\rvc_N\}$, $\theta = \{\rvm_1, ...,\rvm_Q, \mu_1,...,\mu_Q\}$ and $\psi = \emptyset$. Further, let the subscripts for these partitions indicate the corresponding latent factors in FA; i.e., for SPARFA, $\theta_i = [\rvm_i, \mu_i]$ and $\nu_i = \rvc_i$. Since $\psi$ summarizes known factors, we will omit it from the mathematical expositions in the remainder of the paper.

The inference objective in FA is then to obtain a good estimate of the unknown factors, represented by $\theta$ and $\nu$, with respect to some loss function $\mathcal{L}$, usually the marginal data log likelihood. 
There are two ways to infer the unknown factors. 

\subsubsection{Maximum Likelihood Estimation}
The first way is through the maximum likelihood estimate (MLE) which obtains a point estimate of the unknown factors
\begin{align}
    \widehat{\theta},\widehat{\nu} 
    &= \underset{\theta,\nu}{\rm argmin}\,\,\bigg(-\hspace{-5pt}\sum_{i,j\in\Omega_{\rm obs}}\hspace{-5pt} {\rm log}\,p(y_{ij} ; \theta,\nu)\bigg) +  \lambda \mathcal{R}(\theta,\nu) \,.\label{eq:mle}
\end{align}
Recall that $\Omega_{\rm obs}$ indicates which student $i$ has provided an answer to question $j$. $\mathcal{R}(\theta,\nu)$ is a regularization term, e.g., $\ell2$ regularization on the factors and $\lambda$ is a hyper-parameter that controls the strength of the regularization. Most of the works in FA use this method of inference~\cite{LFA,PFA,chi2011instructional,vie2019knowledge,das3h,lindsey2014improving,mozer2016predicting}. The advantage of MLE is that the above optimization objective allows the use of fast inference algorithms, in particular stochastic gradient descent (SGD) and its many variants, making the inference highly efficient.

\subsubsection{Maximum A Posteriori Estimation}
\label{sec:bayes}
The second way is through maximum a posteriori (MAP) estimation through Bayesian inference. This is the focus of our paper. Recall that we wish to obtain not only a point estimate but also credible interval, which MAP allows while MLE does not. 

Bayesian inference methods treat the unknown parameters $\nu$ and $\theta$ as random variables drawn from some distribution. We are thus interested in inferring the \textit{posterior} distribution of $\nu$ and $\theta$. Using the Bayes rule, we have that 
\begin{align}
    p(\nu | \mY, \theta)  &= \frac{p(\mY, \nu, \theta)}{p(\mY)} \label{eq: bayes-rule-1}\\
    &= \frac{p(\mY | \nu, \theta)\,p(\nu)\,p(\theta)}{\iint p(\mY | \nu, \theta)\,p(\nu)\,p(\theta) d\nu d\theta}. \label{eq:bayes-rule-2}
\end{align}
We can similarly derive the posterior for the other unknown factor $\theta$, although in this paper we focus on Bayesian inference for the unknown student skill level parameter $\nu$.
In Eq.~\ref{eq: bayes-rule-1} to Eq.~\ref{eq:bayes-rule-2} we have applied standard Bayes rule. Note that, because $\psi$ contains known factors and thus does not enter the Bayes rule equation, we have omitted it from the above equations. Once we estimate the posterior distribution for $\nu$ from data, we can obtain both point estimates and credible intervals by respectively taking the mean and the standard deviation of the posterior distributions. A number of prior works have proposed to use Bayesian inference for factor analysis, mostly developed for the IRT model~\cite{fox2001bayesian,mislevy1986bayes}. More recently, Bayesian methods are developed for more complex models. For example,~\cite{lan14a} proposed SPARFA-B, a specialized MCMC algorithm that accounts for SPARFA's sparsity and nonnegativity assumptions.

However, Bayesian inference suffers from high computational cost despite its desired capability to quantify uncertainty. The challenge stems from the difficulty to evaluate the denominator in the posterior distribution in Eq.~\ref{eq: bayes-rule-1} and~\ref{eq:bayes-rule-2} which involves an integration over a potentially multi-dimensional variable. This term usually cannot be computed in close form, thus we usually cannot get an analytical formula for the posterior distribution. Monte Carlo Markov Chain (MCMC) method gets around this difficulty by using analytically tractable proposal distributions to approximate the true posterior and sequentially updating the proposal distribution in a manner reminiscent of gradient descent. MCMC methods also enjoy the theoretical advantage that, when left running for long enough, the proposal distribution is guaranteed to converge to the true posterior distribution~\cite{gelman2013bayesian}. However, in practice, it might take very long time for MCMC to converge to the point that running MCMC is no longer practical when data set is very large. The high computational complexity of MCMC is apparently impractical for educational applications where timely feedback is of critical importance~\cite{driscoll1994psychology}. 




\section{V{\normalfont\textbf{ar}}FA: A Variational Inference Factor Analysis Framework}
\label{sec:method}
We are ready to introduce VarFA, our variational inference (VI) factor analysis framework for efficient Bayesian learning analytics. The core idea follows the variational principle, i.e., we use a parametric variational distribution to approximate the true posterior distribution. VarFA is highly flexible and efficient, making it suitable for large scale Bayesian inference for FA models in the context of educational data mining. 

In this current work, we focus on obtaining credible interval for the student skill mastery factor $\nu$ as a first step of VarFA because, recall from Section~\ref{sec:intro}, this factor is often of more practical interest than the other unknown factors. Therefore, currently VarFA is a hybrid MAP and MLE inference method: we perform VI on the unknown student factor $\nu$ and MLE on all other unknown factors $\theta$. We consider this hybrid nature of VarFA in its current formulation as a novel feature because classic MLE or MAP estimation are not capable of performing Bayesian inference on a subset of the unknown factors.
Extension to VarFA to full Bayesian inference for all unknown factors is part of an ongoing research; see~\ref{sec:conclusion} for more discussions.


\subsection{VarFA Details}
\label{sec:method-details}
Now, we explain in detail how to apply variational inference for FA models for efficient Bayesian inference.
Because the posterior distribution is intractable to compute (recall Eq.~\ref{eq: bayes-rule-1} and~\ref{eq:bayes-rule-2} and the related discussion), we approximate the true posterior distribution for $\nu$ with a parametric variational distribution
\begin{align}
    p(\nu | \mY, \theta) \approx q_{\phi}(\nu | \mY) = \prod_{i=1}^N q_{\phi}(\nu_i | \vy_i)  \,,
\end{align}
where $\phi$ is a collection of learnable parameters that parametrize the variational distribution and $\vy_i$ is all the answer records by student $i$. Notably, we have removed the dependency of the variational distribution on $\psi$ and $\theta$ so that the variational distribution is solely controlled by the variational parameter $\phi$. Thus, the design of the variational distribution is highly flexible. All we need to do is to specify a class of distributions and design a function parametrized by $\phi$ to output the parameters of $q_\phi$. 
Common in prior literature is to use a Gaussian with diagonal covariance for $q_\phi$:
\begin{align}
    q_\phi(\nu | \vy_i) = \mathcal{N}(\vu_j, {\rm diag}(\vv_j))\,,\label{eq:variational-distribution}
\end{align}
where its mean and variance $[\vu_j^\top, \vv_j^\top]^\top = f_\phi(\vy_i)$. We can use arbitrarily complex functions such as a deep neural network for $f_\phi$ as long as they are differentiable; more details in Section~\ref{sec:implementation-detail}.
With the above approximation, Bayesian inference turns into an optimization problem under the variational principle, where we now optimize a lower bound, known as the evidence lower bound (ELBO)~\cite{bishop2006pattern}, of the marginal data log likelihood. We derive a novel ELBO objective for variational inference applied to FA models in the EDM setting, summarized in the proposition below.
\begin{prop} In VarFA, the ELBO objective for an FA model in the form of Eq.~\ref{eq:sparfa} is 
\begin{align}
\begin{split}
        \mathcal{L}_{\rm ELBO}(\phi, \theta) =& \sum_{i,j\in\Omega_{\rm obs}} \bigg(-D_{\rm KL}[q_\phi(\nu_i|\vy_i)\|p(\nu_i)] \\
    &\quad + \mathbb{E}_{\nu_i\sim q_\phi(\nu_i|\vy_i)}[{\rm log}\,p_{\theta_j}(y_{ij}|\nu_i)]\bigg)
\end{split}
\label{eq:elbo}
\end{align}
where $D_{\rm KL}$ is the Kullback–Leibler (KL) divergence~\cite{mackay2003information} between two distributions.
\end{prop}
\begin{proof} 
We start with the marginal data log likelihood which is the objective we want to maximize and introduce the random variables $\nu_i$'s: 
\begin{align} 
    {\rm log}\,p(\mY) &= \sum_{i,j\in\Omega_{\rm obs}} {\rm log}\, p_{\theta_j}(y_{ij}) \nonumber \\ 
    &= \sum_{i,j\in\Omega_{\rm obs}} {\rm log}\int p_{\theta_j}(y_{ij}, \nu_j) \,d\nu_i \nonumber \\
    &= \sum_{i,j\in\Omega_{\rm obs}} {\rm log}\int q_\phi(\nu_i|\vy_i) \frac{p_{\theta_j}(y_{ij}, \nu_j)}{q_\phi(\nu_i|\vy_i)}\,d\nu_i \nonumber \\
    &\overset{(i)}{\geq} \sum_{i,j\in\Omega_{\rm obs}}\mathbb{E}_{\nu_i\sim q_\phi(\nu_i|\vy_i)}\left[ {\rm log}\frac{p_{\theta_j}(y_{ij}, \nu_j)}{q_\phi(\nu_i|\vy_i)}\right] \nonumber\\
    &\geq \sum_{i,j\in\Omega_{\rm obs}} \hspace{-5pt}- \mathbb{E}_{\nu_i}\left[{\rm log} \frac{q_\phi(\nu_i|\vy_i)}{p(\nu_j)}\right] + \mathbb{E}_{\nu_i}\left[ {\rm log}\,p_{\theta_j}(y_{ij}|\nu_j)\right] \nonumber\\
    &= \mathcal{L}_{\rm ELBO}(\phi,\theta)  \nonumber
\end{align}
where step (i) follows from Jensen's inequality~\cite{jensen1906fonctions}. \qed
\end{proof}

The ELBO objective differs from those used in existing literature in that, in our case, because the data matrix is only partially observed, the summation is only over $\{i,j\}\in\Omega_{\rm obs}$, i.e., the observed entries in the data matrix. In contrast, in prior literature, the summation in the ELBO objective is over all all entries in the data matrix.

We form the following optimization objective to estimate $\phi$ and $\theta$:
\begin{align}
    \widehat{\theta}, \widehat{\phi} = \underset{\theta, \phi}{\rm argmin}\, -\mathcal{L}_{\rm ELBO}(\phi, \theta) + \lambda \mathcal{R}(\theta)\,,\label{eq:objective}
\end{align}
where $\mathcal{R}(\theta)$ is a regularization term.
That is, we perform VI on the student factor $\nu$ and MLE inference on the remaining factors $\theta$. More explicitly, to perform VI on $\nu$, we simply compute the mean and standard deviation of $q_\phi$ using $f_\phi$ with the learnt variational parameters $\phi$. 

\subsection{Why is VarFA efficient?}
\label{sec:implementation-detail}
\vspace{-10pt}
\paragraph{VarFA supports efficient stochastic optimization}
By formulating Bayesian inference as an optimization problem via the ELBO objective, VarFA allows efficient optimization algorithms such as SGD, with a few additional tricks. In particular, we require all terms in $\mathcal{L}(\theta,\phi)$ to be differentiable with respect to $\theta$ and $\phi$ which enable gradient computation necessary for SGD. This requirement is easily satisfied with a few moderate assumptions. Specifically, we let the variational distribution $q_\phi$ and the prior distribution $p(\nu_i)$ for each $i$ to be Gaussian (See Eq.~\ref{eq:variational-distribution}; we use a standard Gaussian for the prior distribution $p(\nu_i)$, i.e., $p(\nu_i) = \mathcal{N}(0, \mI)$) following existing literature~\cite{kingma2013auto,rezende2014stochastic}. These two assumptions are not particularly limiting especially given the vast number of successful applications of VI that rely on the same assumptions~\cite{infovae,dd,siddharth2017learning,fairvae,betatcvae,betavae,factorvae}.
Thanks to the Gaussian assumption, for the first term in $\mathcal{L}(\theta,\phi)$, we can easily compute the KL divergence term analytically in closed form. For the second term in $\mathcal{L}(\theta,\phi)$, we can use the so-called reparametrization trick, i.e., $\nu_i = \vu_i + \vv_i \odot \bm{\epsilon}$, where $\bm{\epsilon}\sim\mathcal{N}(\bm{0}, \mI)$, which allows a low variance estimation of the gradient of the second term in $\mathcal{L}(\theta,\phi)$ with respect to $\phi$. See Section 2.3 in~\cite{kingma2013auto}, Section 2.3 in~\cite{kingma2019introduction} and Section 3 in~\cite{rezende2014stochastic} for more details on stochastic gradient computation in VI. As a result, our framework can be efficiently implemented using a number of open source, automatic differentiation packages such as Tensorflow~\cite{tensorflow2015-whitepaper} and PyTorch~\cite{NEURIPS2019_9015}.

\vspace{-10pt}
\paragraph{VarFA supports amortized inference} Classic Bayesian inference methods such as MCMC infer each parameter $\nu_i$ for each student. As the number of students increases, the number of parameters that MCMC needs to infer also increases, which may not be scalable in large-scale data settings. In contrast, unlike classic Bayesian inference methods, VI is \textit{amortized}. It does \textit{not} infer each parameter $\nu_i$. Rather, it estimates a \textit{single} set of variational parameter $\phi$ responsible for inferring \textit{all} $\nu_i$'s. In this way, once we have trained FA models using VarFA and obtain the variational parameter $\phi$, we can easily infer $\nu_i$ by simply computing $q_\phi$, even for new students, without invoking any additional optimization or inference procedures. 

\subsection{Dealing with missing entries} Note from Eq.~\ref{eq:elbo} that the variational distribution $q_\phi$ for each $\nu_i$ is conditioned on $\vy_i\in\mathbb{R}^Q$, i.e., an entire row in $\mY$. However, in practice, the data matrix $\mY$ is often only partially observed, i.e., some students only answer a subset of questions. Then, $\vy_i$'s will likely contain missing values which computation cannot be performed on. To work around this issue, we use ``zero imputation'', a simple strategy that transforms the missing values to 0, following prior work on applying VI in the context of recommender systems~\cite{liang2018variational,salakhutdinov2007restricted,sedhain2015autorec,nazabal2018handling} that demonstrated the effectiveness of this strategy despite its simplicity. We note that there exist more elaborate ways to deal with missing entries, such as designing specialized function $f_\phi$ for the variational distribution~\cite{ma2018eddi,gong2019icebreaker,ma2018partial}. We leave the investigation of more effectively dealing with missing entries to future work.

\begin{figure*}[htp!]
    \centering
    \includegraphics[width=0.32\linewidth]{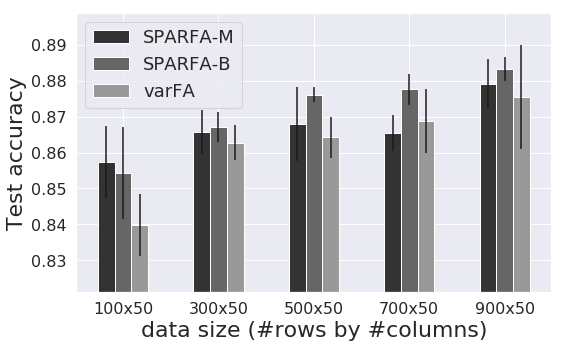}
    \hspace{3pt}
    \includegraphics[width=0.32\linewidth]{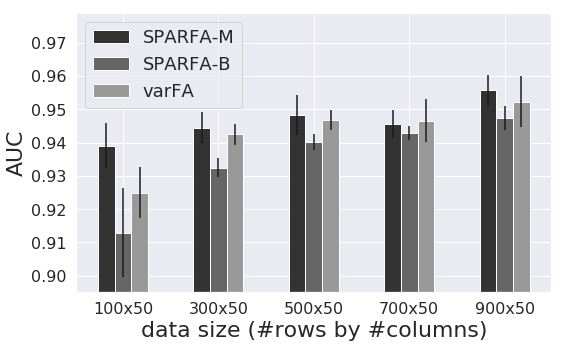}
    \hspace{3pt}
    \includegraphics[width=0.32\linewidth]{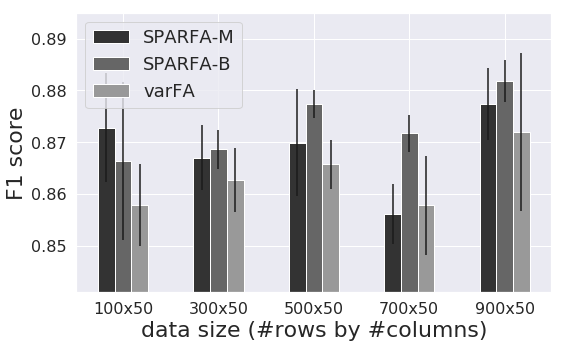}
    \caption{Performance of the  SPARFA-M, SPARFA-B, and VarFA algorithms on the synthetic data set with different data sizes. Plots from left to right show comparison on accuracy (ACC), area under curve (AUC) and F1 metrics, respectively. Higher is better for all metrics. VarFA  performs
    similarly to SPARFA-M and SPARFA-B.}
    \label{fig:synthetic}
\end{figure*}

\subsection{Remarks}
\vspace{-10pt}
\paragraph{Applicability of VarFA} The VarFA  framework is general and flexible and can be applied to a wide array of FA models. The recipe for applying VarFA to an FA model of choice is as follows: 1) formulate the FA model into the canonical formulation as in Eq.~\ref{eq:sparfa}; 2) partition the factors into student factor $\nu$, the remaining unknown factors $\theta$ and known factors $\psi$; 3) perform VI on $\nu$ and MLE estimation on $\theta$, following Section.~\ref{sec:method-details}.

\vspace{-10pt}
\paragraph{Relation to variational auto-encoders}Our proposed framework can be regarded as a standard variational auto-encoder (VAE) but with the decoder implemented not as a neural network but by the FA model. Because the decoder is constraint to a FA model, it is more interpretable than a neural network. 

\vspace{-10pt}
\paragraph{Relation to other efficient Bayesian factor analysis methods}
We acknowledge that VI applied to general FA models have been proposed in existing literature~\cite{ghahramani2000variational,zhao2009note}. However, to our knowledge, little prior work have applied VI to educational FA models nor investigated its effectiveness. One concurrent work applied VI to IRT~\cite{wu2020variational}. Our VarFA framework applies generally to a number of other FA models by following the recipe in the preceding paragraph, including IRT. Thus, our work complements existing literature in providing promising results in applying VI for FA models in the context of educational data mining.

\section{Experiments}
\label{sec:experiments}
We demonstrate the efficacy of VarFA variational inference framework using the sparse factor analysis model (SPARFA) as the underlying FA model. 
This choice of SPARFA as the FA model to investigate is motivated by its mathematical generality: it assumes no auxiliary information is available and all three factors $\rvc_i$'s, $\rvm_j$'s and $\mu_j$'s need to be estimated, which makes the inference problem more challenging.~\cite{lan14a} provides two inference algorithms including SPARFA-M (based on MLE) and SPARFA-B (based on MAP). 

We conduct experiments on both synthetic and real data sets. Using synthetic data sets, we compare VarFA to both SPARFA-M and SPARFA-B. We demonstrate that 1) VarFA predicts students' answers as accurately as SPARFA-M and SPARFA-B; 2) VarFA is almost 100$\times$ faster than SPARFA-B. Using real data sets, we compare VarFA to SPARFA-M. We demonstrate that 1) VarFA predicts students' answers more accurately than SPARFA-M; 2) VarFA can output the same insights as SPARFA-M, including point estimate of students' skill levels and questions' associations with skill tags; 
3) VarFA can additionally output meaningful uncertainty quantification for student skill levels, which SPARFA-M is incapable of, without sacrifice to computational efficiency. Note that SPARFA-B can also compute uncertainty for small data sets but fails for large data sets due to scalability issues and thus we do not compare to SPARFA-B for real data sets.

Specifically for SPARFA, using the notation convention in Section~\ref{sec:FA-inference}, the question-skill association factors $\rvm_j$'s and the question difficulty factors $\mu_j$ are collected in $\theta = \{\rvm_1,...,\rvm_Q,\mu_1,...,\mu_Q\}$. The student skill level factors $\rvc_j$'s are collected in $\nu = \{\rvc_1,...,\rvc_N\}$. Because the factors $\rvm_j$'s are unknown, the ``skills'' in SPARFA are \textit{latent} (referred to as ``latent skills'' in subsequent discussions) and are not attached to any specific interpretation. However, as we will see in Section~\ref{sec:experiment-post-processing}, assuming the skill tags are available as auxiliary information, we can associate the estimated latent skills with the provided skills tags using the same approach proposed in~\cite{lan14a}. For the regularization term in Eq.~\ref{eq:objective}, because the factors $\rvm_j$'s are assumed to be sparse and nonnegative, we use $\ell1$ regularization for $\rvm_j$'s. When performing MLE for $\rvm_j$'s, we use the proximal gradient algorithm for optimization problem with nonnegative and sparsity requirements; see Section 3.2 in~\cite{lan14a} for more details. For the other unknown factor $\mu_j$'s in $\theta$, we apply standard $\ell2$ regularization.
The code along with instructions to reproduce our experiments can be downloaded from~{\color{blue}{\url{https://tinyurl.com/tvm4332}}}.

\begin{figure}
    \centering
    \includegraphics[width=0.9\linewidth]{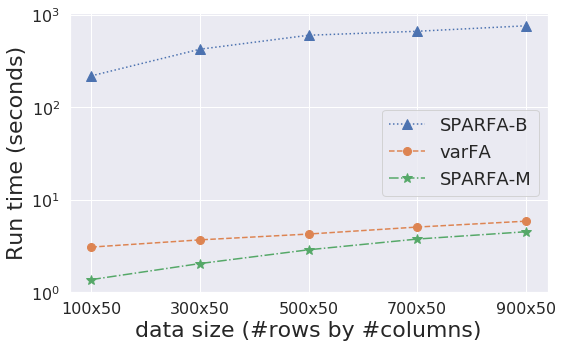}
    \caption{Training run time comparing VarFA, SPARFA-M, and SPARFA-B on synthetic data sets of varying sizes. VarFA performs approximate Bayesian inference almost 100x faster than SPARFA-B and is close to the run time of SPARFA-M. Thus, VarFA enables practical and scalable Bayesian inference for very large data sets.}
    \label{fig:run-time}
\end{figure}
\subsection{Synthetic Data Experiments}
\subsubsection{Setup}
\vspace{-10pt}
\paragraph{Data set generation} We generate data set according to the canonical FA model specified in Eq.~\ref{eq:sparfa}, taking into consideration the additional sparsity and nonnegativity assumptions in SPARFA. Specifically, We sample the factors $\mu_j$'s and $\rvc_i$'s from i.i.d. standard isotropic Gaussian distributions with variance $1$ (or identity matrix for $\rvc_i$'s) and mean sampled from a uniform distribution in range $[-1, 1]$. To simulate sparse and nonnegative factors $\rvm_j$'s, we first sample an auxiliary variable $s_{kj}\stackrel{\rm\tiny i.i.d.}{\sim}{\rm Bernoulli}(\pi)$, where $\pi$ controls the sparsity, and then sample $m_{kj}\stackrel{\rm\tiny i.i.d.}{\sim}{\rm Exponential}(1)$ when $s_{kj}=1$ or setting $m_{kj}=0$ when $s_{kj}=0$ for each $j$ and $k$. In all synthetic data experiments, we use $\pi=0.3$ and set the true number of latent skills to $K=5$. 

\vspace{-10pt}
\paragraph{Experimental settings} 
We vary the size of the data set and use 5 different data matrix sizes: 100$\times$50, 300$\times$50, 500$\times$50, 700$\times$50 and 900$\times$50. We select a data missing rate of 50\%, i.e., we randomly choose 50\% of all entries in the data matrix as training set and the rest as test set. 
Experimental results for each of the above data sizes are averaged over 5 runs where we randomize over the train/test data split. 
We train SPARFA with both SPARFA-M and VarFA using the Adam optimizer~\cite{kingma2014adam} with learning rate $=0.05$ for 100 epochs. Regularization hyper-parameters are chosen by grid search for each experiment. For VarFA, we use a simple 3-layer neural network for the function $f_\phi$ in the variational distribution. 
We use the hyper-parameter settings for SPARFA-B following Section 4.2 in~\cite{lan14a}. 


\begin{table}[]
\caption{Summary statistics of pre-processed real data sets.}
\vspace{-5pt}
\begin{tabular}{lccc}
\toprule
\textbf{data set} & \textbf{\#students} & \textbf{\#questions} & \textbf{\%observed} \\ \hline
assistment        & 392                  & 747                   & 12.93\%                      \\
algebra           & 697                  & 782                   & 13.02\%                      \\
bridge            & 913                  & 1242                  & 11.42\%                      \\ \bottomrule
\end{tabular}
\label{tab:data-stat}
\end{table}

\vspace{-10pt}
\paragraph{Evaluation metrics} We evaluate the inference algorithms on their ability to recover (predict) the missing entries in the data matrix given the observed entries. We term this criterion ``student answer prediction''. Since our data matrix is binary, using prediction accuracy (ACC) alone does not accurately reflect the performance of the inference algorithms under comparison. Thus, we use area under the receiver operating characteristic curve (AUC) and F1 score in addition to ACC, as standard in evaluating binary predictions~\cite{hastie2009elements}.

\subsubsection{Results} 
Fig.~\ref{fig:synthetic} shows bar plots that compare the 
student answer prediction performance
of VarFA with SPARFA-M and SPARFA-B on all three evaluation metrics. 
We can observe that all three methods perform similarly and that SPARFA-M and SPARFA-B show no statistically significant advantage over VarFA.
\sloppy
We further showcase VarFA's scalability by comparing its training run-time with SPARFA-M and SPARFA-B. The results are shown in Fig.~\ref{fig:run-time}. VarFA is significantly faster than SPARFA-B and is almost as fast as SPARFA-M. 

In summary, with VarFA, we obtain posteriors that allow uncertainty quantification with roughly the \textit{same} computation complexity of computing point estimates and \textit{without} compromising prediction performance. 

\begin{table}[t!]
\caption{Student answer prediction erformance comapring VarFA to SPARFA-M on Assistment, Algebra and Bridge data sets. $\uparrow$ and $\downarrow$ denote higher and lower is better, respectively. VarFA performs better than SPARFA-M on all three data sets and evaluation metrics most of the time. Additionally, VarFA's run time is very close to SPARFA-M.}
\centering
\begin{subtable}{\linewidth}\centering
\caption{Assistment}
\vspace{-5pt}
\begin{tabular}{lcc}
\toprule
\textbf{Metric} & \multicolumn{2}{c}{\textbf{Algorithm}} \\  \hline \\[-8pt]
                & SPARFA-M           & VarFA \\ 
                \cmidrule(l){2-2} \cmidrule(l){3-3}
ACC $\uparrow$  & 0.7074$\pm$0.0044  & \textbf{0.7101}$\pm$0.0048 \\
AUC $\uparrow$            & 0.756$\pm$0.048   & \textbf{0.7635}$\pm$0.0036 \\
F1 $\uparrow$            & 0.7746$\pm$0.0029  & \textbf{0.7765}$\pm$0.0014 \\ \hline
Run time (s) $\downarrow$& \textbf{5.3319}$\pm0.2774$  & 6.9167$\pm$0.1074 \\
\bottomrule
\vspace{1pt}
\end{tabular}
\label{tab:assistiemnt}
\end{subtable}
\begin{subtable}{\linewidth}\centering
\caption{Algebra}
\vspace{-5pt}
\begin{tabular}{lcc}
\toprule
\textbf{Metric} & \multicolumn{2}{c}{\textbf{Algorithm}} \\  \hline \\[-8pt]
                & SPARFA-M           & VarFA \\ 
                \cmidrule(l){2-2} \cmidrule(l){3-3}
ACC $\uparrow$        & 0.7735$\pm$0.0037  & \textbf{0.7774}$\pm$0.0031 \\
AUC $\uparrow$           & 0.8137$\pm$0.003   & \textbf{0.8245}$\pm$0.002 \\
F1 $\uparrow$            & 0.8465$\pm$0.0021  & \textbf{0.8486}$\pm$0.001 \\\hline
Run time (s) $\downarrow$ & \textbf{8.464}$\pm$0.4568 & 10.3335$\pm$0.4435  \\
\bottomrule
\vspace{1pt}
\end{tabular}
\label{tab:algebra}
\end{subtable}
\begin{subtable}{\linewidth}\centering
\caption{Bridge}
\vspace{-5pt}
\begin{tabular}{lcc}
\toprule
\textbf{Metric} & \multicolumn{2}{c}{\textbf{Algorithm}} \\  \hline \\[-8pt]
                & SPARFA-M           & VarFA \\ 
                \cmidrule(l){2-2} \cmidrule(l){3-3}
ACC $\uparrow$       & \textbf{0.8492}$\pm$0.0016  & 0.8468$\pm$0.0016 \\
AUC $\uparrow$            & 0.837$\pm$0.0024   & \textbf{0.8419}$\pm$0.0028 \\
F1 $\uparrow$             & \textbf{0.9121}$\pm$0.0005  & 0.912$\pm$0.0009 \\\hline
Run time (s) $\downarrow$ & \textbf{15.6048}$\pm$0.7314 & 15.8558$\pm$1.046  \\ 
\bottomrule
\end{tabular}
\label{tab:bridge}
\end{subtable}
\label{tab:real-data}
\end{table}

\begin{figure*}[t!]
    \centering
    \begin{subfigure}[b]{0.31\textwidth}
    \includegraphics[width=\linewidth]{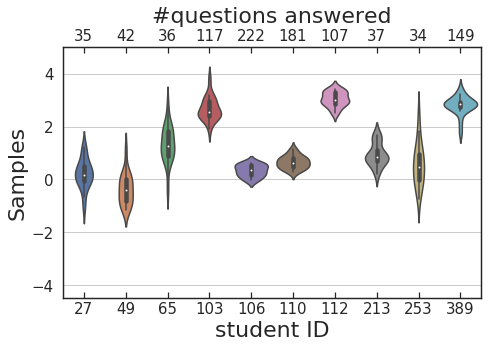}
    \caption{3rd latent concept}
    \label{fig:violin1}
    \end{subfigure}
    \hspace{5pt}
    \begin{subfigure}[b]{0.31\textwidth}
    \includegraphics[width=\linewidth]{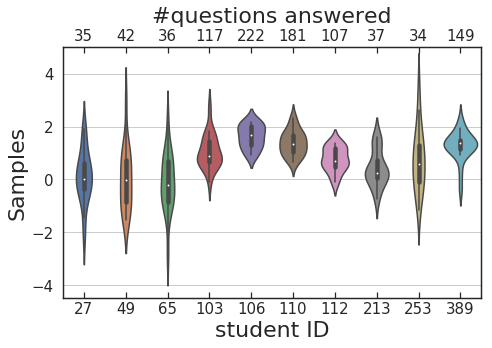}
    \caption{4th latent concept}
    \label{fig:violin2}
    \end{subfigure}
    \hspace{5pt}
    \begin{subfigure}[b]{0.31\textwidth}
    \includegraphics[width=\linewidth]{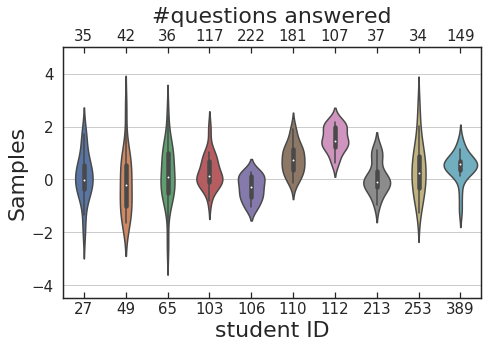}
    \caption{7th latent concept}
    \label{fig:violin3}
    \end{subfigure}
    \caption{Violin plot showing the mean and standard deviation of the estimated skill mastery levels on 10 selected students on the 3rd, 4th and 7th latent skills that VarFA computes. In each sub-figure, bottom and top axises respectively shows student IDs and top axis shows the number of questions each student answered. The more questions a student answers, the tighter the credible interval. (Best viewed in color.)}
    \label{fig:violin}
\end{figure*}

\subsection{Real Data Experiments}
\subsubsection{Setup}
\vspace{-10pt}
\paragraph{Data sets and pre-processing steps} We perform experiments on three large-scale, publicly available, real educational data sets including ASSISTments 
2009-2010 (Assistment)~\cite{assistment}, Algebra I 2006-2007 (algebra)~\cite{Algebra} and Bridge to Algebra 2006-2007 (bridge)~\cite{bridge,Stamper2016The2K}. The details of the data sets, including data format and data collection procedure can be found in the preceding references. We remove students and questions that have too few answer records from the data sets to reduce the sparsity of the data sets. Specifically, we keep students and questions that have no less than 30, 35 and 40 answer records for Assistment, Algebra and Bridge data sets, respectively. 
In each data set, each student may provide more than one answer record for each question. Therefore, we also remove student-question answer records except for the first one.
Table~\ref{tab:data-stat} presents the summary statistics of the resulting pre-processed data matrix for each data set. 

\vspace{-10pt}
\paragraph{Experimental settings and evaluation metrics} Optimizer, learning rate, number of epochs, neural network architecture and evaluation metrics are the same as in synthetic data experiments. For real data sets, we use 8 latent skills instead of 5. Other choices of the number of latent skills might result in better performance. However, since we are comparing different inference algorithms for the \textit{same} model, it is a fair to compare the inference algorithms on the same model with the same number of latent dimensions.
We use a 80:20 data split, i.e., we randomly sample 80\% and 20\% of the observed entries in each data set, without replacement, as training and test sets, respectively. Regularization hyper-parameters are selected by grid search and are different for each data set. We only compare VarFA with SPARFA-M because SPARFA-B does not scale to such large data sets. 

\subsubsection{Results: Performance Comparison}
Table~\ref{tab:real-data} shows the average performance on the test set of each data set comparing VarFA and SPARFA-M for all three data sets and additionally run time. We can see that VarFA achieves slightly better student answer prediction on most data sets and on most metrics. 
Interestingly, recall that, in the preceding synthetic data set experiment, SPARFA-M performed better than VarFA most of the time. 
A possible explanation is that, for synthetic data sets, the underlying data generation process matches the SPARFA model, whereas for real data sets, the underlying data generation process is unknown.
VarFA's better performance than SPARFA-M for real data sets implies that VarFA is more robust to the unknown underlying data generation process. As a result, VarFA may be more applicable than SPARFA-M in real-world situations.

Table~\ref{tab:real-data} also shows the run time comparison between VarFA and SPARFA-M; see the last row in each sub-table. We see that both inference algorithms have very similar run time, showing that VarFA is applicable for very large data sets. Notably, VarFA achieves this efficiency while also performing Bayesian inference on the student knowledge level factor. 

\subsubsection{Results: Bayesian Inference With VarFA}
We now illustrate VarFA's capability of outputting credible intervals using the Assistment data set.
Fig.~\ref{fig:violin} presents violin plots that show the sampled student latent skill levels for a random subset of 10 students. Plots~\ref{fig:violin1},~\ref{fig:violin2} and~\ref{fig:violin3} shows the inferred students ability for the 3rd, 4th and 7th latent skill dimension. In each plot, the bottom axis shows the student ID and the top axis shows the total number of questions answered by the corresponding student. For each student, the horizontal width of the violin represents the density of the samples; the skinnier the violin, the more widespread the samples are, implying the model's less certainty on its estimations.

\begin{table*}[t!]
\caption{Illustration of the estimated latent skills with the their top 3 most strongly associated skill tags in the Assistment data set. The percentage in the parenthesis shows the association probability (summed to 1 for each latent skill). We see that the tagged skills associated with each estimated latent skill form intuitive and interpretable groups. 
}
\centering
\setlength\heavyrulewidth{0.25ex}
\begin{tabular}{ll}
\toprule
\textbf{Latent Skill 1}                                                                                                                                               & \textbf{Latent Skill 3}                                                                                                                            \\ \hline \\[-5pt]
\begin{tabular}[c]{@{}l@{}}Division Fractions (29.1\%)\\ Least Common Multiple (18.1\%)\\ Write Linear Equation from Ordered Pairs (17.8\%)\end{tabular}                    & \begin{tabular}[c]{@{}l@{}}Conversion of Fraction Decimals Percents (7.3\%)\\ Addition and Subtraction Positive Decimals (6.8\%)\\ Probability of a Single Event (5.7\%)\end{tabular} \vspace{10pt}\\ \toprule \\[-10pt]
\textbf{Latent Skill 4}                                                                                                                                               & \textbf{Latent Skill 7}
\\ \hline \\[-5pt]
\begin{tabular}[c]{@{}l@{}}Pattern Finding (17.4\%)\\ Histogram as Table or Graph (11.3\%)\\ Percent Of (10.5\%)\end{tabular} & \begin{tabular}[c]{@{}l@{}}Volume Sphere (13.4\%)\\ Volume Cylinder (10.4\%)\\  Surface Area Rectangular Prism (10.2\%)\end{tabular}        \\ \bottomrule
\end{tabular}
\end{table*}

Results in Fig.~\ref{fig:violin} confirms our intuition that the more questions a student answers, the more certain the model is about its estimation. For example, students with ID 106, 110 and 389 answered 222, 181 and 149 questions, respectively, and the credible intervals of their ability estimation is quite small. In contrast, students with ID 27, 49 and 65 answered far less questions and the credible intervals of their ability estimation is quite large. This result implies that VarFA outputs sensible and interpretable credible intervals. As mentioned in Section~\ref{sec:intro}, such uncertainty quantification may benefit a number of educational applications such as improving adaptive testing algorithms. 
Note that VarFA is able to compute such credible interval as fast as SPARFA-M, making VarFA potentially useful for even real-time educational systems. In contrast, SPARFA-M is not capable of computing credible intervals. Although we may still obtain confidence interval as uncertainty quantification with SPARFA-M via bootstrapping, i.e., train with SPARFA-M multiple times with random subsets of the data set and then use the point estimates from different random runs to compute confidence intervals. However, this method suffers from two immediate drawbacks: 1) it is slow because we need to run the model multiple times and 2) different runs result in permutation in the estimated factors and thus the averaged estimations loss meaning. Given that VarFA is as efficient as SPARFA-M and the previous drawbacks of SPARFA-M, we recommend VarFA over bootstrapping with SPARFA-M for quantifying model's uncertainty in practice.


\subsubsection{Results: Post-Processing for Improved Interpretability}
\label{sec:experiment-post-processing}
SPARFA assumes that each student factor $\nu_i$ identifies a multi-dimensional skill level on a number of ``latent'' skills (recall that we use 8 latent skills in our experiments). As mentioned earlier, these latent skills are not interpretable without the aid of additional information. To improve interpretability,~\cite{lan14a} proposed that, when the skill tags for each question is available in the data set, we can associate each latent skill with skill tags via a simple matrix factorization. Then, we can compute each students' mastery levels on the actual skill tags. We refer readers to Sections 5.1 and 5.2 in~\cite{lan14a} for more technical details. 
Although the above method to improve interpretability has already been proposed,~\cite{lan14a} only presented results on private data sets, whereas here we presents results on publicly available data set with code, making our results more transparent and reproducible.

We again use the Assistment data set for illustration. We compute the association of skill tags in the data set
with each of the latent skills and show 4 of the latent skills with their top 3 most strongly associated skill tags. We can see that each latent skill roughly identify the same group of skill tags. For example, latent skill 4 clusters skill tags on statistics and probability while latent skill 7 clusters skill tags on 
geometry. Thus, by simple post-processing, we obtain an interpretation of the latent skills by associating them with known skill tags in the data.

We can similarly obtain VarFA's estimations of the students' mastery levels on each skill tags through the above process. In Fig.~\ref{fig:skill}, we compare the predicted mastery level for each skill tag (only for the questions this student answered) with the percent of correct answers for that skill tag.
Blue curve shows the empirical student's mastery level on a skill tag by computing the percentage of correctly answered questions belonging to a particular skill tag. Orange curve shows VarFA's estimated student mastery level on a skill tag, normalized to range $[0,1]$. 
We can see that, when the student gave more correct answers to questions of a particular skill tag, such as skill tag ID 2, 10, 14 and 30, VarFA also predicts a higher ability score for these skill tags. When the student gave more incorrect answers to questions of a particular skill tag, such as skill tag ID 0, 4 and 27, VarFA also predicts a lower ability score for those skill tags. Although the correspondence is not perfect, VarFA's predicted student's mastery levels match our intuition about student's abilities (using the observed correct answer ratio as an empirical estimate) reasonably well. Thus, by post-processing, we can interpret VarFA's estimated students' latent skill mastery levels using the easily understandable skill tags. 

\section{Conclusions and Future Work}
\label{sec:conclusion}
We have presented VarFA, a variational inference factor analysis framework to perform efficient Bayesian inference for learning analytics. VarFA is general and can be applied to a wide array of FA models. We have demonstrated the effectiveness of our VarFA using the sparse factor analysis (SPARFA) model as a case study. We have shown that VarFA can very efficiently output interpretable, educationally meaningful information, in particular credible intervals, much faster than classic Bayesian inference methods. Thus, VarFA has potential application in many educational data mining scenarios where efficient credible interval computation is desired, i.e., in adaptive testing and adaptive learning systems. We have also provided open-source code to reproduce our results and facilitate further research efforts.

We outline three possible future research directions and extensions. First,
VarFA currently performs Bayesian inference on the student skill mastery level factor. We are working on extending the framework to perform full Bayesian inference for all unknown factors. Extensions to some models such as IRT is straightforward, as studied in~\cite{wu2020variational}, because all factors can be reasonably assumed to be Gaussian. 
However, some other models involve additional modeling assumptions, making it challenging to design distributions that satisfy these assumptions. For example, in SPARFA, one of the factors is assumed to be nonnegative and sparse~\cite{lan14a}. No standard distribution fulfill these requirements. We are exploring \textit{mixture} distributions, in particular spike and slab models~\cite{ishwaran2005spike} for Bayesian variable selection, and methods to combine them with variational inference, following~\cite{titsias2011spike,enlighten191553}.

Second, other methods exist that accelerate classic Bayesian inference. Recent works in large-scale Bayesian inference proposed \textit{approximate} MCMC methods that scale to large data sets~\cite{zhang2020csgmcmc,ma2015complete,song2017nice}. Some of these methods apply variational inference to perform the approximation~\cite{de2001variational,habib2018auxiliary}. We are investigating extending VarFA to support approximate MCMC methods.

\begin{figure}
    \centering
    \includegraphics[width=\linewidth]{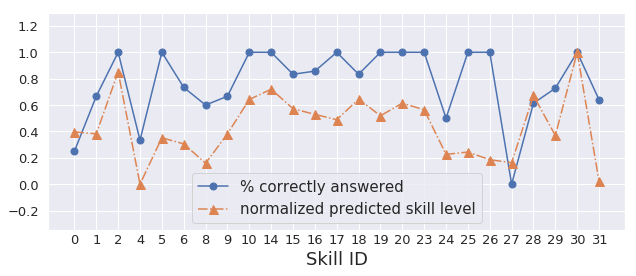}
    \vspace{-15pt}
    \caption{Comparison between the estimated skill mastery levels using VarFA's predictions and using empirical observations for student with ID 110. Even though the two curves show different numeric values, they nevertheless demonstrate similar trends, showing that the predictions reasonably match our intuition about student's skill mastery levels. 
    }
    \label{fig:skill}
    \vspace{-5pt}
\end{figure}

Finally, the goal of VarFA is to improve real-world educational systems and ultimately improve learning. Therefore, it is necessary to evaluate VarFA beyond synthetic and benchmark data sets. We plan to integrate VarFA into an existing educational system and conduct a case study where students interact with the system in real-time to understand the VarFA's educational implications in the wild.

\newpage

\section*{Acknowledgements}
This work was supported by
NSF grants CCF-1911094, IIS-1838177, IIS-1730574, DRL-1631556, IUSE-1842378;
ONR grants N00014-18-12571 and N00014-17-1-2551;
AFOSR grant FA9550-18-1-0478; and a
Vannevar Bush Faculty Fellowship, ONR grant N00014-18-1-2047.

%
\bibliographystyle{abbrv}
\bibliography{sigproc}  
%
%

\end{document}

%% file: preamble.tex
\usepackage{amsmath,amsfonts,bm, bbm}
\usepackage{amssymb}
\usepackage{mathtools}

\def\eqref#1{equation~\ref{#1}}

\def\1{\bm{1}}

\def\rvc{{\mathbf{c}}}

\def\rvm{{\mathbf{m}}}

\def\vu{{\bm{u}}}
\def\vv{{\bm{v}}}

\def\vy{{\bm{y}}}

\def\mI{{\bm{I}}}

\def\mQ{{\bm{Q}}}

\def\mY{{\bm{Y}}}

\DeclareMathAlphabet{\mathsfit}{\encodingdefault}{\sfdefault}{m}{sl}
\SetMathAlphabet{\mathsfit}{bold}{\encodingdefault}{\sfdefault}{bx}{n}

\newcommand{\sigmoid}{\sigma}

